%% file: main.tex
\documentclass[11pt]{article}
\usepackage{arxiv} 

\usepackage[utf8]{inputenc}
\usepackage[T1]{fontenc}
\usepackage{amsmath,amssymb,amsthm}
\usepackage{graphicx}
\usepackage{booktabs}
\usepackage{microtype}
\usepackage[hidelinks]{hyperref}
\hypersetup{pdftitle={Epistemic Sybil Resistance: Multiplying AI Agents Without Multiplying Evidence},
            pdfauthor={Marc Bara}}

\newtheorem{theorem}{Theorem}
\newtheorem{proposition}{Proposition}
\newtheorem{corollary}{Corollary}
\theoremstyle{definition}
\newtheorem{definition}{Definition}

\allowdisplaybreaks

\title{Epistemic Sybil Resistance:\\Multiplying AI Agents Without Multiplying Evidence}
\author{Marc Bara\\Universitat Oberta de Catalunya (UOC), Barcelona, Spain\\\texttt{mbara@uoc.edu}}
\date{August 2026}

\begin{document}

\maketitle

\begin{abstract}
Multi-agent AI systems improve inference by spawning agents and synthesizing reports. But another agent is not another observation: apparently independent reports may descend from the same evidence, and genuinely independent evidence can produce nearly identical reports. We formalize this as an epistemic Sybil problem. A report $Z$ is an epistemic Sybil extension relative to reports $R$ when $I(\Theta;Z\mid R)=0$. No report-only aggregator can generally distinguish replication from independent corroboration: identical reports can warrant different posteriors under unobserved ancestry. A Gaussian shared-root model shows common ancestry does not imply complete redundancy. Repeated extraction adds information toward a source-level ceiling, and correlated extraction errors, which a shared base model can induce among independent agents, lower that ceiling further. We test these predictions with more than 20,000 controlled LLM-agent report and extraction calls on synthetic evidentiary documents. Holding one evidence root fixed while report multiplicity rises from 1 to 32 collapses naive posterior coverage from 0.940 to 0.263. Holding report count fixed while evidence-root multiplicity rises from 1 to 16 closes the gap, and the aggregators are statistically indistinguishable at $k=16$. The agent's replicate extraction errors are correlated ($\hat\gamma_{\mathrm{cal}}=0.719$, estimated out of sample), and a correlated-extraction aggregator restores calibration accordingly. A controlled manipulation isolates representation similarity from evidential ancestry. It changes a report-space deduplication mechanism's mean inferred cluster count by $1.425$ ($95\%$ CI $[1.363,1.485]$), whereas a fourfold change in true ancestry changes it by only $0.040$ ($[-0.045,0.120]$). Collective inference should therefore track evidential ancestry and dependence, not agent or report multiplicity or similarity.
\end{abstract}

\section{Introduction}

Suppose three analysts predict that a firm will default. One has examined its audited accounts, another has access to payment data from its banks, and a third has independently interviewed major suppliers. Now consider three other analysts who produce different and detailed arguments but ultimately derive them from the same credit-rating report.

Both situations contain three reports and may generate the same vote or confidence level. They nevertheless represent different information structures.

The distinction matters increasingly because generating reports has become cheap, and increasingly because the inference system itself can now generate them. An orchestrator can send one source to several agents, pass one agent's output into another agent's input, expose several agents to overlapping retrieval results, or combine outputs from repeated instances of a shared base model; a single source can be summarized, translated, critiqued, or reformulated before being re-aggregated, returning to the system in forms that appear unrelated. Conversely, genuinely independent evidence may produce near-identical reports. Observable diversity and evidential diversity therefore need not coincide.

Many classical fusion models treat dependence as a property of pre-existing sources, although distributed-estimation and data-incest settings already show that inference architectures can also create dependence through retransmission. Generative agent systems amplify this second regime: spawning an agent, branching a retrieval, or re-prompting a shared model are actions the system itself takes, not properties of a fixed external source population, so an agentic system is not merely aggregating under dependence; its own orchestration can create the dependence that it subsequently mistakes for corroboration. Agent multiplicity, report multiplicity, and evidence-root multiplicity are consequently different quantities, and the central question is when an additional agent's report constitutes new evidence rather than merely a new voice. We formalize its marginal evidential contribution as the conditional information it adds about the latent state.

Dependence among information sources is itself a long-studied problem in forecast combination, distributed estimation, and social learning; Section~\ref{sec:related-work} situates this paper relative to that literature and to recent multi-agent AI systems that explicitly trace or discount evidence ancestry. The problem addressed here is the architectural form this takes in collective AI inference: when a system can cheaply multiply agents, reasoning paths, and reports, what quantity should govern evidential accumulation, and what must an aggregator know to prevent apparent corroboration from growing faster than the evidence that supports it? We treat marginal information about the latent state, rather than agent or report multiplicity, as the quantity that should remain invariant under such replication. The question is native to autonomous agents and multi-agent systems rather than merely illustrated by them: belief aggregation, trust and reputation mechanisms, and human-agent collectives each presuppose an answer to when one more agent's report warrants more collective confidence, and generative orchestration now forces that presupposition to be made explicit and enforced by the architecture itself.

Let $\Theta$ denote a latent state and let
\[
R=(R_1,\ldots,R_n)
\]
denote reports already admitted by an inference system. We call an additional report $Z$ an epistemic Sybil extension when
\[
I(\Theta;Z\mid R)=0.
\]
The report may have a different identity, use different language, present a different argument, or pass through an independent computational process. The defining property is that, conditional on what has already been admitted, it contributes no additional information about $\Theta$.

This reframes Sybil resistance. Classical Sybil attacks multiply identities. An epistemic Sybil multiplies apparent evidential support without multiplying information about the state. The phenomenon can be adversarial, but it can also arise through honest retransmission.

Two questions organize the rest of the paper. First, can report content alone identify evidential ancestry well enough to enforce this invariant? Section~\ref{sec:non-identifiability} shows it cannot in general. Second, when reports do share a root, how much information does that root actually contain, and how does that change when the reports' extraction errors are themselves correlated, as they can be for agents sharing a base model? Section~\ref{sec:shared-roots} answers both questions in a Gaussian model. Section~\ref{sec:provenance} then asks what a provenance interface can and cannot certify. Sections~\ref{sec:simulation} and~\ref{sec:empirical-study} test these results computationally and empirically: a synthetic Monte Carlo study validates the analytical model under a dependence structure imposed by construction, and more than 20,000 calls to a real language-model agent show the same qualitative failure, and its correction, under dependence that is measured rather than assumed, including a controlled design that separates representation similarity from evidential ancestry directly.

The paper makes four contributions.
\begin{enumerate}
\item It introduces epistemic Sybil resistance (Definition~\ref{def:esr-interface}) as a criterion for collective inference and identifies the protected quantity as marginal conditional information about the latent state, rather than agent identity, report count, or report similarity. The formulation makes explicit that evidential contribution is meaningful only relative to the information available to the aggregator.
\item It establishes a report-only identification barrier: the same observable report profile can correspond to replicated ancestry or independent corroboration and require different Bayes-optimal updates, while a complementary no-minting result bounds the information that descendants of fixed evidence can collectively contain.
\item It derives how information accumulates between the extremes of exact replication and independent evidence. Repeated extraction from a shared root adds information only toward a finite ceiling, and correlated extraction errors of the kind shared models can induce lower that ceiling further, separating evidential ancestry from the dependence introduced by the extraction process itself.
\item It tests these distinctions in more than 20,000 calls to a real language-model agent. At fixed ancestry, multiplying reports produces severe overconfidence under independence-assuming aggregation; increasing the number of genuinely independent roots removes that failure; replicate extraction errors exhibit substantial residual correlation; and a controlled $2\times2$ design shows that a concrete report-space deduplication mechanism responds strongly to representation while barely responding to the evidential ancestry it is intended to recover.
\end{enumerate}

Taken together, these contributions turn a classical dependence problem into an explicit design problem for collective AI inference. Generative orchestration can transform one evidential lineage into many agents, reports, reasoning paths, and representations at negligible marginal cost, making nominal multiplicity an increasingly unreliable proxy for evidential multiplicity. The contribution here is to make that distinction operational: agent, report, and evidence-root multiplicity are separated explicitly; marginal evidential contribution is defined through conditional information; the consequences of hiding ancestry and dependence from the aggregator are isolated formally; shared-root dependence is distinguished from correlated extraction; and the resulting predictions are tested under controlled report and ancestry manipulations with real language-model calls. The underlying phenomena of dependent evidence are classical; what the framework supplies is a common criterion for deciding when additional AI-generated voices constitute additional evidence.

\section{Related Work}
\label{sec:related-work}

The literature below establishes that dependence among information sources undermines naive aggregation across forecast combination, distributed estimation, social learning, and covariance-based fusion. What generative multi-agent orchestration changes is the scale, cost, and representational flexibility with which an inference system can itself multiply and transform reports.

\subsection{Dependent information aggregation}

Dependence among information sources has long been recognized as central to forecast combination, and further back, to the epistemology of testimony and instrumentation: Olsson (2002) and Bovens and Hartmann (2002) study probabilistically when agreement among witnesses or repeated instrument readings constitutes genuine corroboration, placing the same independence question within formal epistemology. Closer still to the causal structure used throughout this paper, Dietrich and List (2004) model jury decisions in which each juror's vote descends not from the true state itself but from a shared body of evidence, precisely the $\Theta\rightarrow E\rightarrow(R_1,\ldots,R_n)$ structure of Section~\ref{sec:esr-reports}; they show that once this shared root is acknowledged, the classical Condorcet jury theorem's guarantee of near-certain collective correctness as the jury grows can fail, and evidence can mislead even maximally competent jurors. Winkler (1981) studies the combination of probability distributions derived from dependent information sources, and Clemen and Winkler (1999, 2007) develop this into a general treatment of dependent expert aggregation: shared backgrounds, methods, or data sources among experts produce redundant rather than additional information, dependence must be modeled as an input to aggregation separate from each expert's marginal reliability, and experts who are informationally equivalent given a common source contribute little beyond either alone. Kuncheva, Whitaker, Shipp, and Duin (2003) document a parallel effect in classifier fusion: pools of classifiers with identical individual accuracy can achieve substantially different majority-vote accuracy under different joint dependence structures, while pairwise diversity measures alone do not uniquely determine ensemble performance. Distributed estimation provides a close precursor to the present problem: when estimates are repeatedly exchanged across a network, information can return to an estimator embedded in another estimate, and treating the returned estimate as independent counts some observations repeatedly, producing overconfident uncertainty (McLaughlin, Evans, and Krishnamurthy, 2003); data integration research studies an analogous problem when sources copy from one another, so that agreement between dependent sources should not be treated as the independent corroboration a naive combination rule would infer from it (Das Sarma, Dong, and Halevy, 2011). Covariance Intersection addresses a related problem when cross-correlations are unknown (Julier and Uhlmann, 1997), and recent work on Overlapping Covariance Intersection extends this to settings with partial structural knowledge of shared correlation (Pedroso, Batista, and Heemels, 2026); Section~\ref{sec:discussion} returns to this connection. Economics provides a behavioral counterpart: DeMarzo, Vayanos, and Zwiebel (2003) show how repeated transmission can create influence unrelated to informational accuracy, Golub and Jackson (2010) characterize when naive social learning nevertheless aggregates information successfully, and Enke and Zimmermann (2019) provide experimental evidence that people systematically neglect correlation and double count information derived from common sources. An epistemic Sybil can be read as a strategic or algorithmic version of this behavioral phenomenon, sharpened by generative systems that reduce the cost of retelling while making common ancestry harder to infer from surface form.

This literature establishes an important novelty boundary: this paper does not claim that dependence among reports, experts, sources, or classifiers undermines naive aggregation, that shared sources create redundant evidence, or that a shared evidential root can defeat classical convergence guarantees for collective judgment (Dietrich and List, 2004); each is already established. Dietrich and List's model, like the jury-theorem literature generally, takes the shared-root structure as given to the analyst. This paper's focus is different: the fact that this structure is not generally observable from reports themselves (Theorem~\ref{thm:non-identifiability}), how information degrades further when the process extracting from an acknowledged root is itself correlated across reports (Section~\ref{sec:correlated-extraction}), and a controlled empirical test of both under real language-model agents.

\subsection{Sybil and clone robustness}

Douceur (2002) introduced the classical Sybil attack as the creation of multiple identities by one participant; false-name manipulation in mechanism design studies the corresponding possibility that one economic actor presents itself as multiple participants. Burnat and Davidson (2026) apply this perspective to machine-learning data attribution, using quotient mechanisms that aggregate over evidence-backed attribution clusters rather than raw identities. Berriaud and Wattenhofer (2026) formalize clone robustness for arbitrary metric spaces, preventing exact or near-duplicate reports from obtaining additional aggregate influence merely through multiplication. The protected resource here is different from either: epistemic Sybil resistance concerns neither identity weight nor payment attribution but marginal information relevant to a belief, and metric clone robustness is related but non-equivalent, since identical reports can result from independent observations while distant reports can be transformations of a single observation. Theorem~\ref{thm:non-identifiability} formalizes why report-space representation alone cannot universally distinguish these cases.

\subsection{Provenance-aware aggregation}

Human-agent collective systems provide an early application-level precedent for the general problem. Ramchurn et al. (2016) build a disaster-response system that explicitly discusses duplicated/common-source information, combines heterogeneous reports using machine-learning methods, and separately tracks provenance across the system to support accountability and downstream decisions. The present paper addresses a different formal question: it supplies a latent-state conditional-information definition of evidential redundancy, a report-only non-identifiability result, a shared-root information ceiling with a correlated-extraction generalization, and a distinction between evidential ancestry and the statistical dependence strength it induces, together with a controlled empirical test separating ancestry from representation similarity, none of which that system formalizes or evaluates directly.

Recent AI research increasingly records evidence lineage. Yan et al. (2026) trace evidence roots while studying misinformation propagation in multi-agent systems; Xu (2026) uses content-addressed lineage as a global side channel in decentralized mesh inference; Louck (2026) binds authority to origin and preserves it through subsequent transformations; Lee and Kim (2026) demonstrate a complementary failure in retrieval-augmented generation, where lexically diverse passages manufacture corroboration while evading near-duplicate filters. Wu (2026) is particularly close to the present statistical problem: sources sharing an upstream origin are placed in dependence blocks, and evidence from a block of size $m$ is discounted using $\kappa_m=1/[1+\rho(m-1)]$, treated as an effective-sample-size approximation. Section~\ref{sec:shared-roots} shows this expression is exact within a homoscedastic Gaussian equicorrelation model, giving the coefficient a direct generative interpretation. Peer-prediction mechanisms raise a related tension: the Mutual Information Paradigm of Kong and Schoenebeck (2019) rewards informative reports using information-theoretic dependence between peers, but correlation between reports can arise either because independent signals are informative about a common state or because one report derives from another, so a mechanism that rewards agreement without identifying its causal origin can reward endogenous dependence. Truthful provenance is a distinct mechanism-design problem from truthful content elicitation; Section~\ref{sec:discussion} returns to it.

\subsection{Information-theoretic connections}

The conditional-information definition used throughout also admits familiar interpretations through Blackwell informativeness and partial information decomposition; these connections are useful for positioning but are not contributions of the paper. If $\Theta\rightarrow E\rightarrow Y$, then $Y$ is a garbling of $E$, and observing $E$ Blackwell-dominates observing $Y$: no decision maker can obtain lower Bayes risk from the garbled observation for any decision problem satisfying the standard comparison-of-experiments conditions (Blackwell, 1953). In a nonnegative two-source partial information decomposition, conditional mutual information can be written as the sum of information unique to $Z$ and information available only synergistically from $R$ and $Z$, so $I(\Theta;Z\mid R)=0$ means an epistemic Sybil contributes neither unique nor synergistic information about $\Theta$ beyond $R$ (Williams and Beer, 2010).

\section{Epistemic Sybil Resistance}
\label{sec:esr}

\subsection{Reports, evidence roots, and marginal information}
\label{sec:esr-reports}

Let $\Theta$ be a latent random variable on state space $\mathcal T$. An inference system receives $R=(R_1,\ldots,R_n)$ and produces a belief $A(R)\in\Delta(\mathcal T)$. A report may originate from a sensor, human, model, database, document, retrieval operation, algorithm, or any composition of them.

Two reports $X$ and $Z$ satisfy the standard conditional-independence model when $X\perp Z\mid\Theta$; agreement between them can then add substantial evidence. They can instead derive from a common evidence variable $E$, so that $\Theta\rightarrow E\rightarrow(X,Z)$. We call such a shared upstream evidence variable $E$ an evidence root. Common ancestry induces dependence, but does not imply that one report is redundant conditional on the other, since different reports can extract different aspects of the same evidence.

The limiting form of redundancy is $\Theta\perp Z\mid R$, equivalently $I(\Theta;Z\mid R)=0$.

\begin{definition}[Epistemic Sybil extension]
Given an admitted report profile $R$, an additional report $Z$ is an epistemic Sybil extension relative to $R$ if
\[
I(\Theta;Z\mid R)=0.
\]
A collection $Z_{1:k}$ is an epistemic Sybil extension if
\[
I(\Theta;Z_{1:k}\mid R)=0.
\]
\end{definition}

The definition is invariant to representation: a Sybil extension need not duplicate an existing string, embedding, identity, argument, or source label. Define its incremental epistemic contribution as $\Gamma(Z;R)=I(\Theta;Z\mid R)$, so $\Gamma(Z;R)=0$ for an epistemic Sybil extension and $\Gamma(Z;R)>0$ for a report containing marginal information about $\Theta$. The quantity is a theoretical target, not assumed directly observable.

\subsection{Side-information interfaces}
\label{sec:interfaces}

Epistemic Sybil resistance cannot be defined independently of what the aggregator is allowed to observe. Let $\mathcal P$ be a class of information-generating models. A side-information interface $\mathcal J$ maps a model $P\in\mathcal P$ and an observed report profile $R$ to side information $S=\mathcal J_P(R)$ available to the aggregation mechanism, ranging from $\mathcal J_0$ (no structural side information) through $\mathcal J_G$ (authenticated provenance information) to $\mathcal J_*$ (the complete relevant joint information structure). An aggregation mechanism relative to $\mathcal J$ is written $A_{\mathcal J}(R,\mathcal J_P(R))$.

\begin{definition}[Epistemic Sybil resistance relative to an interface]
\label{def:esr-interface}
An aggregation mechanism $A_{\mathcal J}$ is epistemic-Sybil-resistant on model class $\mathcal P$ relative to side-information interface $\mathcal J$ if, whenever $P\in\mathcal P$ and an extension $Z$ satisfies
\[
I_P(\Theta;Z\mid R)=0,
\]
then
\[
A_{\mathcal J}\left(R,Z,\mathcal J_P(R,Z)\right)
=
A_{\mathcal J}\left(R,\mathcal J_P(R)\right)
\]
almost surely.
\end{definition}

Resistance is therefore always relative to an information interface: Section~\ref{sec:non-identifiability} establishes a limitation under $\mathcal J_0$, and Section~\ref{sec:provenance} examines what richer interfaces provide. The defining conditional-independence relation also has a sufficiency interpretation: if $I(\Theta;Z\mid R)=0$, then $R$ remains sufficient for $\Theta$ after the observation space is enlarged from $R$ to $(R,Z)$.

In practice, interfaces of increasing strength reveal increasingly specific structure: a report-only interface exposes content but nothing about ancestry; a root interface additionally exposes authenticated root sets $\{\mathcal R(R_i)\}_{i=1}^n$, identifying known shared ancestry without determining the strength of statistical dependence it induces; a structural interface exposes the derivation graph or predicates over it, establishing facts such as ``this report derives only from already-admitted reports'' without identifying unobserved statistical dependencies; and a probabilistic interface augments lineage with an explicit generative dependence model, at which point aggregation reduces to ordinary probabilistic inference under that model. The design relation is therefore provenance structure plus dependence assumptions yields evidential contribution; provenance structure alone should not be interpreted as evidential weight, a point Section~\ref{sec:provenance} develops further.

\section{No-Minting and Report-Only Non-Identifiability}
\label{sec:no-minting-and-nonid}

\subsection{No-minting principle}
\label{sec:no-minting}

Let $E$ represent primitive evidence and let $Y=(Y_1,\ldots,Y_m)$ be any collection of reports generated solely from $E$, so that $\Theta\rightarrow E\rightarrow Y$.

\begin{proposition}[Evidence no-minting]
\[
I(\Theta;Y)\leq I(\Theta;E).
\]
\end{proposition}

\begin{proof}
The result follows directly from the data-processing inequality. The corresponding Blackwell result (Blackwell, 1953) is stronger: since $Y$ is a garbling of $E$, observing $E$ is at least as useful as observing $Y$ for every Bayesian decision problem under the standard comparison-of-experiments framework.
\end{proof}

The proposition gives a source-level conservation principle: processing can expose information already contained in evidence, but multiplying its descendants cannot create information about the underlying state that was absent from their ancestry. Importantly, this does not imply that every descendant of $E$ after the first should be discarded. A first report can reveal only part of the information latent in $E$; additional processing can expose more. The total information in all descendants remains bounded by the information present in their ancestry. This distinction separates exact replication from repeated extraction, the subject of Section~\ref{sec:shared-roots}.

\subsection{Report-only non-identifiability}
\label{sec:non-identifiability}

Classical work establishes that dependence is a separate input to Bayesian aggregation and that posterior conclusions can be sensitive to it (Winkler, 1981; Clemen and Winkler, 1999). Marginal source reliabilities alone do not determine the joint conditional likelihood, so the information available to the aggregator matters directly for what posterior can be justified. Theorem~\ref{thm:non-identifiability} turns this dependence sensitivity into a report-only identification result: when dependence or ancestry is hidden from the aggregator, the same observable report profile can require different Bayes-optimal updates, and no rule that sees only report content can supply both.

Let $\Theta\sim\operatorname{Bernoulli}(1/2)$, and let a binary report $X$ satisfy $P(X=\Theta)=p$, $p\in(1/2,1)$. Consider two possible information structures.

Under the \emph{clone structure}, the second report is an exact copy, $Z=X$, so upon observing $X=Z=1$,
\[
P_C(\Theta=1\mid X=1,Z=1)=p.
\]

Under the \emph{independent structure}, instead $X\perp Z\mid\Theta$ and both reports have accuracy $p$, so
\[
P_I(\Theta=1\mid X=1,Z=1)
=
\frac{p^2}{p^2+(1-p)^2}
>p
\qquad(p>1/2).
\]
The observable report profile is identical in the two structures, while the appropriate posterior differs.

\begin{theorem}[Report-only non-identifiability]
\label{thm:non-identifiability}
Consider an aggregation rule operating under $\mathcal J_0$ and observing only $(X,Z)$ together with the common marginal accuracy $p$. No deterministic rule can be Bayes-correct under both the clone and independent information structures.

If its posterior estimate after observing $(1,1)$ is $q$, then
\[
\max\left\{
|q-p|,
\left|
q-\frac{p^2}{p^2+(1-p)^2}
\right|
\right\}
\geq
\delta(p),
\]
where
\[
\delta(p)
=
\frac{p(1-p)(2p-1)}
{2[p^2+(1-p)^2]}.
\]
The same lower bound holds for randomized rules in expected absolute error.
\end{theorem}

The mechanism in Theorem~\ref{thm:non-identifiability} is deliberately granted knowledge of $p$; withholding the marginal reliability would only make identification more difficult. For $p=0.7$, $q_C=0.700$, $q_I\approx0.845$, and $\delta(0.7)\approx0.072$: a gap of that size is unavoidable for any single deterministic response to the same observable input. The proof, a short computation of the Bayes-optimal posteriors in each structure and the resulting bound, is given in Appendix~\ref{app:proof-nonidentifiability}.

\begin{corollary}[Incompatibility under the trivial interface]
Over a model class containing both clone and independent-corroboration structures, no rule operating under $\mathcal J_0$ can guarantee both epistemic-Sybil invariance and Bayes-correct responsiveness to independent corroboration.
\end{corollary}

The result does not imply that semantic similarity is useless. Similarity can be statistically informative about common ancestry. It cannot, however, provide a universally valid identification criterion; Section~\ref{sec:emp-2x2} tests this directly for one concrete report-space mechanism.

Theorem~\ref{thm:non-identifiability} isolates the report-only identification barrier relevant to epistemic Sybil resistance: once the dependence structure is hidden, correct aggregation cannot in general be recovered from report content and marginal reliability alone. This is a different result from ensemble-dependence findings such as Kuncheva, Whitaker, Shipp, and Duin (2003), which study how joint correctness dependence among classifiers changes realized ensemble accuracy for a fixed decision rule. Theorem~\ref{thm:non-identifiability} instead fixes the observable report profile and varies the hidden information structure, showing why additional structural information is required to distinguish replicated support from independent corroboration.

\section{Information Aggregation under Shared Roots}
\label{sec:shared-roots}

The opposite error to Section~\ref{sec:non-identifiability}'s is to conclude that every set of reports sharing one evidence root should count as exactly one observation.

\subsection{Shared-root model}
\label{sec:shared-root-model}

Suppose
\[
E=\Theta+\varepsilon,
\qquad
\varepsilon\sim\mathcal N(0,\sigma^2),
\]
and $m$ reports independently extract information from $E$:
\[
R_i=E+\eta_i,
\qquad
\eta_i\sim\mathcal N(0,\nu^2).
\]
Initially assume that the $\eta_i$ are mutually independent and independent of $\varepsilon$ and $\Theta$. Conditional on $\Theta$, $R\sim\mathcal N(\Theta\mathbf 1,\Sigma_m)$ with $\Sigma_m=\sigma^2\mathbf 1\mathbf 1^\top+\nu^2I$, and the likelihood precision about $\Theta$ is $J_m=\mathbf 1^\top\Sigma_m^{-1}\mathbf 1$.

\begin{proposition}[Shared-root saturation]
\label{prop:saturation}
For $\nu^2>0$,
\[
J_m
=
\frac{m}{\nu^2+m\sigma^2}
=
\frac{1}{\sigma^2+\nu^2/m}.
\]
Hence
\[
\lim_{m\rightarrow\infty}J_m
=
\frac{1}{\sigma^2}.
\]
The marginal precision generated by the $m$-th report is
\[
J_m-J_{m-1}
=
\frac{\nu^2}
{(\nu^2+m\sigma^2)[\nu^2+(m-1)\sigma^2]},
\]
which is positive and of order $O(m^{-2})$.
\end{proposition}

\begin{proof}
By the Sherman--Morrison identity (Sherman and Morrison, 1950),
\[
\Sigma_m^{-1}
=
\frac{1}{\nu^2}I
-
\frac{\sigma^2}
{\nu^2(\nu^2+m\sigma^2)}
\,\mathbf 1\mathbf 1^\top.
\]
Consequently, $J_m=\mathbf 1^\top\Sigma_m^{-1}\mathbf 1=m/(\nu^2+m\sigma^2)$. The remaining expressions follow directly.
\end{proof}

\subsection{Scaling regimes and effective evidence}
\label{sec:scaling-regimes}

The model separates three cases. For exact replication, $\nu^2=0$, so every report is the same observation $E$ and $J_m=1/\sigma^2$: repetition adds no information. For repeated extraction with $\nu^2>0$, additional reports reduce extraction noise, but their contribution saturates because the common error $\varepsilon$ cannot be averaged away. If each report instead has an independent root, $E_i=\Theta+\varepsilon_i$ with independent $\varepsilon_i$, then $J_m^{\mathrm{ind}}=m/(\sigma^2+\nu^2)$: precision grows linearly with report count. Exact replication, shared-root extraction, and independent corroboration therefore have fundamentally different information scaling.

Let $v=\sigma^2+\nu^2$ be the conditional variance of an individual report. The intraclass correlation conditional on $\Theta$ is $\rho=\sigma^2/(\sigma^2+\nu^2)$. Since $\nu^2+m\sigma^2=v[1+\rho(m-1)]$, Proposition~\ref{prop:saturation} can be written as
\[
J_m
=
\frac{m}{1+\rho(m-1)}
\,\frac{1}{v}.
\]
Define $m_{\mathrm{eff}}=m/[1+\rho(m-1)]$, the classical survey-sampling design effect (Kish, 1965) applied here to a likelihood rather than a variance of a sample mean: the $m$ correlated reports contain exactly the same likelihood precision as $m_{\mathrm{eff}}$ independent reports with variance $v$.

\begin{corollary}[Exact effective-sample-size discount under Gaussian equicorrelation]
\label{cor:m-eff}
Within the shared-root Gaussian model,
\[
\kappa_m
=
\frac{1}{1+\rho(m-1)}
\]
is exactly $m_{\mathrm{eff}}/m$.
\end{corollary}

The dependence discount used by Wu (2026) therefore coincides with the exact precision correction implied by this model. The result is specific to the homoscedastic Gaussian equicorrelation class; applying the same coefficient to arbitrary evidence structures is not generally Bayes-optimal. Within the model, $\rho$ has a direct interpretation: it is the proportion of each report's conditional variance attributable to the common source component.

\subsection{Correlated extraction errors}
\label{sec:correlated-extraction}

Independence among extraction errors can be unrealistic when several AI agents share the same base model. Suppose $\operatorname{Var}(\eta_i)=\nu^2$ and, for $i\neq j$, $\operatorname{Cov}(\eta_i,\eta_j)=\gamma\nu^2$ with $0\leq\gamma<1$. The total conditional variance remains $v=\sigma^2+\nu^2$, while off-diagonal covariance becomes $c=\sigma^2+\gamma\nu^2$, giving intraclass correlation $\rho_{\mathrm{tot}}=(\sigma^2+\gamma\nu^2)/(\sigma^2+\nu^2)$. Hence
\[
J_m
=
\frac{m}
{v[1+\rho_{\mathrm{tot}}(m-1)]},
\]
and when $c>0$,
\[
\lim_{m\rightarrow\infty}J_m
=
\frac{1}{\sigma^2+\gamma\nu^2}.
\]
Shared model error therefore introduces an additional information ceiling: nominal agent multiplicity or prompt diversity cannot remove error components common to all extractors. Section~\ref{sec:emp-gamma} estimates $\gamma$ directly from a real language-model agent's replicate extractions and finds this ceiling is not merely a theoretical possibility.

\subsection{Mutual information and the information ceiling}
\label{sec:mutual-info}

If $\Theta\sim\mathcal N(0,\tau^2)$, then $I(\Theta;R_1,\ldots,R_m)=\frac12\log(1+\tau^2J_m)$. For independent extraction from one shared root,
\[
I_{\mathrm{shared}}(m)
=
\frac12
\log\left(
1+
\frac{\tau^2m}{\nu^2+m\sigma^2}
\right),
\]
so
\[
\lim_{m\rightarrow\infty}
I_{\mathrm{shared}}(m)
=
\frac12
\log\left(
1+\frac{\tau^2}{\sigma^2}
\right),
\]
the information available from direct observation of the primitive source $E$.

\section{Provenance as Side Information}
\label{sec:provenance}

Theorem~\ref{thm:non-identifiability} implies that the information-generating structure cannot generally be recovered from report content alone. A richer interface can reveal relevant parts of that structure, but the following limits bound what it can establish.

\subsection{What provenance identifies}
\label{sec:provenance-identifies}

Let $G=(V,E_G)$ be a directed acyclic provenance graph, in the spirit of provenance-as-dependency formalizations such as Cheney, Ahmed, and Acar (2011), whose nodes can represent primitive observations, datasets, documents, transformations, retrieval operations, models, agents, and reports, and whose edges $u\rightarrow v$ record that information from $u$ was available in generating $v$. For report $R_i$, let $\mathcal R(R_i)$ denote its authenticated primitive root set.

Suppose the provenance model establishes $Z=f(R,U)$ with $U\perp\Theta\mid R$. Then $\Theta\perp Z\mid R$ and $Z$ is an epistemic Sybil extension.

\begin{proposition}[Posterior invariance]
If $I(\Theta;Z\mid R)=0$, then $P(\Theta\mid R,Z)=P(\Theta\mid R)$ almost surely.
\end{proposition}

\begin{proof}
Zero conditional mutual information implies the required conditional independence, yielding the posterior identity.
\end{proof}

Under a sufficiently informative interface, ordinary Bayesian updating is therefore automatically invariant to certified zero-novelty extensions; the difficult problem is identifying the relevant dependence structure in the first place.

\subsection{What provenance does not identify}
\label{sec:provenance-limits}

Even when $\mathcal R(R_i)\cap\mathcal R(R_j)=\varnothing$, it need not follow that $R_i\perp R_j\mid\Theta$: distinct datasets can sample the same population, independent sensors can share environmental noise, and apparently independent documents can depend on an unrecorded common source. A provenance interface can establish explicit derivation structure; it cannot prove the absence of every latent common cause.

Two further limits are conceptual rather than statistical. Provenance need not be truthful: participants with an incentive to appear independent can suppress or falsify declared ancestry, so a complete mechanism must address truthful provenance reporting alongside dependence-aware aggregation, not assume it away. And provenance need not be public to be useful, or fully disclosable when it is available: parties may need to establish structural facts about ancestry, such as root-set disjointness, without revealing the roots themselves. Section~\ref{sec:discussion} discusses both as open problems rather than developing them further here.

\subsection{Model parameters as evidence}
\label{sec:model-params}

In generative AI systems, the phrase ``no additional external evidence'' is insufficient. Suppose an LLM produces $Z=f(R,W,U)$, where $W$ denotes model parameters. Those parameters encode regularities acquired during training and can contain information about $\Theta$ that is absent from $R$; in general, $I(\Theta;W\mid R)$ need not be zero. A model instructed only to paraphrase a report can therefore fail to be a pure transformation channel if it corrects, supplements, or reinterprets the report using parametric knowledge. To certify zero novelty one needs an assumption such as $\Theta\perp(W,U)\mid R$, or an explicit modeling convention under which the model's parametric state is not treated as additional admissible evidence for the task. Provenance must therefore characterize relevant information access, not merely external document retrieval; the empirical study of Section~\ref{sec:empirical-study} takes this constraint literally, using only synthetic fictional worlds designed so that the latent state is unavailable from pretraining, precisely so that model parameters cannot act as an unaccounted evidence root.

\section{Synthetic Model Validation}
\label{sec:simulation}

Before testing the framework against real agents, we verify that the analytical results of Section~\ref{sec:shared-roots} are correctly implemented, and establish a quantitative benchmark for the calibration failure they predict, under a dependence structure imposed by construction rather than measured.

We simulate $\Theta\sim\mathcal N(0,\tau^2)$ with baseline parameters $\tau^2=\sigma^2=\nu^2=1$, primitive root $E_j=\Theta+\varepsilon_j$, and reports $R_{ij}=E_j+\eta_{ij}$, varying total report count $n\in\{1,2,4,8,16,32\}$ and root count $k\leq n$, with reports allocated as evenly as possible across roots. Each $(n,k)$ cell contains 60,000 Monte Carlo realizations at a fixed seed. A dependence-aware Bayesian aggregator uses the correct shared-root covariance structure; a naive aggregator treats all $n$ reports as conditionally independent observations of variance $\sigma^2+\nu^2$. No parameter is fitted to the simulated outcomes: the exercise evaluates the consequences of the assumed information structures, not a model fit to reproduce them. Coverage is the fraction of realizations in which the true latent state lies within the nominal 95\% credible interval.

Holding $k=1$ while increasing $n$:

\begin{center}
\begin{tabular}{rcccc}
\toprule
Reports $n$ & Bayes coverage & Naive coverage & Bayes NLL & Naive NLL\\
\midrule
1 & 0.949 & 0.949 & 1.217 & 1.217\\
2 & 0.950 & 0.921 & 1.159 & 1.192\\
4 & 0.949 & 0.833 & 1.128 & 1.378\\
8 & 0.950 & 0.686 & 1.099 & 1.999\\
16 & 0.950 & 0.521 & 1.091 & 3.673\\
32 & 0.950 & 0.381 & 1.077 & 7.273\\
\bottomrule
\end{tabular}
\end{center}

\begin{figure}[t]
\centering
\includegraphics[width=.6\linewidth]{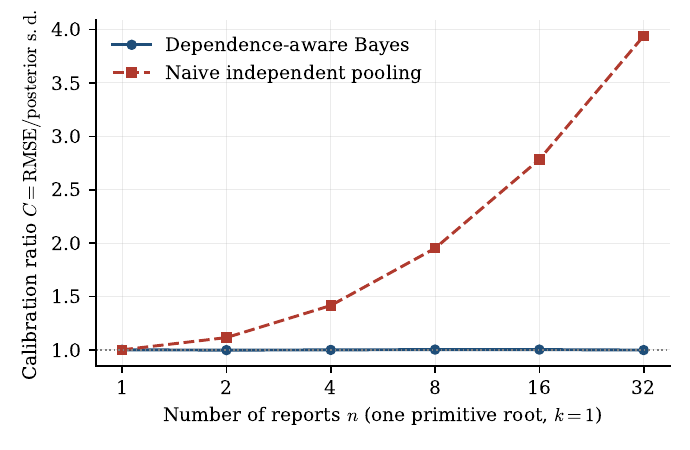}
\caption{Calibration ratio as report count increases with one primitive evidence root. The dependence-aware posterior remains close to one; the independence-assuming posterior becomes progressively overconfident.}
\end{figure}

At $n=32$, naive posterior variance shrinks as though 32 independent observations had arrived: nominal 95\% coverage falls to 38.1\% while the dependence-aware posterior stays near 95\% throughout, and the naive negative log score rises from 1.22 to 7.27 while the dependence-aware score improves slightly. \label{sec:sim-rootcount}The same distinction appears from the other axis: fixing $n=16$ and varying root count $k\in\{1,2,4,8,16\}$, the true information content ranges from 0.332 nats at $k=1$ to 1.099 nats at $k=16$, while a naive aggregator reacting only to $n$ assigns 1.099 nats regardless of $k$, and likelihood precision under one shared root approaches a finite ceiling $1/\sigma^2$ while precision under independent roots keeps growing linearly (Proposition~\ref{prop:saturation}).

The simulation is a validation exercise, not empirical evidence: the dependence structure is imposed, not measured, and it should not be read as a claim that real agents follow the Gaussian model. Its role is to confirm the analytical model is correctly implemented and to provide a quantitative benchmark before Section~\ref{sec:empirical-study} tests the same qualitative predictions against a real language-model agent, where dependence is measured rather than assumed. Additional metrics, the full report-count-versus-root-count and log-score results, and their figures are given in Appendix~\ref{app:simulation-detail}.

\input{draft_empirical_section.tex}

\section{Discussion}
\label{sec:discussion}

Collective inference systems commonly count agents, documents, models, votes, or reports. None of these quantities is generally equivalent to the amount of independent evidence available about a latent state: many agents can carry one informational root, one agent can possess several independent observations, reports that share a root can still contain partially distinct information, and reports with different root labels can remain statistically dependent because of latent common causes. The quantity $I(\Theta;Z\mid R)$ captures the relevant distinction at the probabilistic level, measuring the information about $\Theta$ that becomes available from $Z$ after everything already represented by $R$ is known.

In agentic inference, an orchestrator can branch one source across several agents, pass one agent's report into another agent's input, expose several agents to overlapping retrieval results, or combine outputs from repeated instances of a common base model. Related problems of endogenous retransmission already arise in distributed estimation and data-incest settings, discussed in Section~\ref{sec:related-work}; generative orchestration does not create this dependence problem, but changes its scale, cost, adaptivity, and observability, since one evidential lineage can now be cheaply transformed into many agents, reports, reasoning paths, and superficially diverse representations, so the system can manufacture apparent corroboration faster than it acquires information. Endogenous dependence of this kind is consequently not unique to generative systems; it is already addressed, in a different form, by the distributed-fusion and data-incest literatures cited above, which also study dependence created by the estimation architecture itself rather than inherited from pre-existing source relationships.

The experiments isolate mechanisms directly relevant to these architectures. Section~\ref{sec:emp-multiplicity} shows that report multiplicity at fixed evidential ancestry produces severe overconfidence when reports are treated as independent; in an orchestrated system, agent or message multiplicity can create the same failure when those nominally distinct outputs descend from shared evidence. A report-space deduplication layer, of the kind increasingly proposed as a defense against exactly this failure, can exhibit the representation-tracking failure documented in Section~\ref{sec:emp-2x2}. Provenance-aware aggregation is not a complete fix on its own: Section~\ref{sec:emp-gamma}'s residual miscalibration under a simple shared-root model, and its correction under a correlated-extraction model, show that even reports sharing a known root can have different dependence structures depending on the extraction process, so provenance supplies structural information that report content alone cannot provide, but is not sufficient by itself. It remains a categorically different failure mode from either baseline, and one the theory in Section~\ref{sec:shared-roots} anticipates rather than merely accommodates after the fact. A closely related prediction, that an aggregator with correct ancestry information should outperform a report-only aggregator specifically when one evidential root is retransmitted through a chain of agents, was not tested empirically here and is left for future work alongside the mechanism-design problems below.

The experiments show directly that report multiplicity and representational diversity can rise without corresponding evidential multiplicity: Sections~\ref{sec:emp-multiplicity} and~\ref{sec:emp-2x2} manipulate exactly these two quantities. In orchestrated systems, agent count and message count can become additional nominal multiplicities for the same underlying reason, though this paper does not manipulate them directly. A reliable orchestrator should therefore not treat any of these counts, or their semantic diversity, as a direct proxy for independent evidence. What such an architecture needs to preserve instead is some representation of evidential lineage, of overlap between what different agents actually observed, and of the dependence such overlap implies, close to what Section~\ref{sec:provenance} formalizes as a side-information interface. Correct collective inference, in this sense, is evidence ancestry combined with an appropriate dependence model; neither alone suffices, since ancestry without a dependence model cannot say how much a shared root discounts subsequent reports, and a dependence model without ancestry has nothing to condition on.

This perspective clarifies the relationship among several existing approaches. Classical Sybil resistance protects identity systems against artificial participant multiplicity; clone robustness limits influence created by report duplication or near-duplication; data-incest methods prevent repeated use of information in distributed estimation; Covariance Intersection handles numerical fusion when cross-correlation is unknown; correlation-neglect research studies the behavioral consequences of treating dependent information as independent; and provenance systems expose derivation structure. The epistemic Sybil problem sits at the intersection of these ideas because its protected resource is marginal information about a latent state rather than identities or representations. The Gaussian model additionally shows that a binary duplicate/non-duplicate distinction is insufficient: for equicorrelated reports the effective evidence count $m_{\mathrm{eff}}=m/[1+\rho(m-1)]$ moves continuously, as $\rho$ varies from zero toward one, from independent corroboration toward effective replication. The general problem is therefore not deduplication; it is dependence-aware aggregation under incomplete information about how reports were generated, and Sections~\ref{sec:emp-gamma} and~\ref{sec:emp-2x2} show two concrete ways that incompleteness bites in practice: the wrong dependence model (independent rather than correlated extraction) and the wrong identification signal (representation rather than ancestry).

When dependence is suspected but not fully characterized, robust fusion methods become relevant. Suppose numerical estimates have known marginal covariance matrices but unknown cross-covariance; treating them as independent yields an overconfident posterior, while Covariance Intersection instead chooses a fused representation that remains consistent for the admissible unknown correlations, and Overlapping Covariance Intersection shows that partial structural knowledge of shared correlation can improve on fully correlation-agnostic fusion. This suggests three regimes: known independence permits full accumulation, known dependence calls for model-based aggregation as in Section~\ref{sec:shared-roots}, and unknown dependence calls for conservative aggregation. Evidence provenance can supply exactly the partial structural knowledge such conservative methods exploit; this paper does not propose a replacement for covariance-intersection methods, only the informational invariance criterion they should respect.

The conceptual limits of Section~\ref{sec:provenance-limits} are genuine open problems rather than incidental caveats. If independent evidence receives greater marginal influence, an agent benefits from suppressing ancestry, and a coalition can attempt to manufacture apparently separate roots; if a system rewards agreement, endogenous copying can create precisely the correlation that some information-elicitation mechanisms interpret as evidence of informativeness. A complete mechanism would need to address content reports, provenance reports, and dependence-aware aggregation jointly, seeking conditions under which no participant benefits from misreporting provenance; this remains an open mechanism-design problem. Confidentiality raises a related, largely orthogonal problem: a journalist may wish to establish that two reports originate from separate sources without identifying either source, and firms or agents may need to demonstrate non-overlapping evidential ancestry without exposing proprietary or private material. Cryptographic side-information interfaces, in which participants commit to provenance structures and prove predicates such as $\mathcal R(R_i)\cap\mathcal R(R_j)=\varnothing$ or bounded overlap without disclosing root identities, are one candidate; origin-bound authority and content-addressed lineage already appear in recent agent-security and collective-inference systems, which makes such an interface technically less speculative than the underlying epistemological problem might suggest. Both problems are left for future work.

\section{Limitations}
\label{sec:limitations}

Section~\ref{sec:emp-summary} details limitations specific to the empirical study, including its single task family, single main model, and the representation-manipulation framing of Section~\ref{sec:emp-2x2}; these apply in addition to the following, which concern the framework generally.

First, conditional mutual information is defined relative to a probability distribution. In many real inference problems the joint distribution is unknown. Second, provenance records known derivation, not every latent common cause. Third, model parameters, human background knowledge, and other unobserved states can constitute evidence even when no additional document or sensor is consulted. Fourth, provenance authenticity does not establish the truth of primitive evidence. Fifth, primitive evidence roots can themselves be Sybil-manipulated unless their creation is constrained by domain-specific authentication, resource, hardware, institutional, or identity assumptions. Sixth, the Gaussian model of Section~\ref{sec:shared-roots} is deliberately stylized; its purpose is to distinguish exact duplication, shared-root extraction, and independent corroboration, and to derive closed-form dependence corrections, not to serve as a universal model of natural-language evidence, though Section~\ref{sec:emp-gamma}'s out-of-sample validation suggests it captures a real regularity in at least one concrete setting. Finally, this paper does not provide a universal estimator of $I(\Theta;Z\mid R)$: without structural assumptions or side information, Theorem~\ref{thm:non-identifiability} shows why no universal report-only solution should be expected.

\section{Conclusion}
\label{sec:conclusion}

As agent orchestration makes it cheap to spawn additional reasoning paths, retrievals, and reports, confidence cannot safely scale with the number of apparent contributors: report multiplicity ceases to be a reliable proxy for evidential multiplicity precisely as producing reports becomes cheapest to automate. This paper formalizes the resulting problem through conditional mutual information: an additional report $Z$ is an epistemic Sybil extension relative to admitted reports $R$ when $I(\Theta;Z\mid R)=0$, a definition independent of agent identity and report representation.

Taken together, the theoretical and empirical results support four conclusions. First, the relevant information structure cannot generally be identified from reports alone: Theorem~\ref{thm:non-identifiability} shows that two identical reports can represent either pure replication or genuinely independent corroboration and require different Bayesian treatment, and Section~\ref{sec:emp-2x2} shows a concrete report-space mechanism failing in exactly this way on real data, tracking a representation-only manipulation ($+1.425$) far more than the fourfold change in true ancestry it was meant to detect ($+0.040$). Second, descendants of fixed evidence cannot collectively contain more information about the latent state than their evidential ancestors; report multiplication cannot mint information, by the data-processing inequality. Third, common ancestry does not imply complete redundancy: multiple noisy extractions from one source add genuine information toward a source-level precision ceiling $1/\sigma^2$, with diminishing returns given exactly by the effective-sample-size discount $1/[1+\rho(m-1)]$ under Gaussian equicorrelation, and this ceiling is lower still, $1/(\sigma^2+\gamma\nu^2)$, when extraction errors are correlated across reports, as they measurably are for a real language-model agent extracting from the same document ($\hat\gamma_{\mathrm{cal}}=0.719$, estimated out of sample). Fourth, more than 20,000 calls to a real agent confirm the practical consequence directly: naive aggregation's nominal coverage collapses from 0.940 to 0.263 as report count rises from 1 to 32 at fixed ancestry, and this is not a property of report count as such, since the collapse disappears exactly as ancestry becomes genuinely independent.

These results shift the central question from how many reports agree to how much conditionally novel evidence the reports collectively contain, and from whether reports look alike to whether they were generated independently: this paper's results show these are not the same question.

The next theoretical problem is consequently precise:

\begin{quote}
``What is the weakest side-information interface under which an aggregator can recover, or approximate with bounded loss, the posterior available under the full information-generating structure?''
\end{quote}

A solution would characterize the minimum provenance or dependence information required for reliable collective inference. Truthful provenance, privacy-preserving provenance, and aggregation under partially known dependence, discussed in Section~\ref{sec:discussion}, would then become connected subproblems of that single question rather than separate engineering concerns.

Epistemic Sybil resistance is therefore not a proposal to count sources, or agents, more carefully. It is a proposal to make marginal information, rather than agent multiplicity, report multiplicity, or report similarity, the invariant of collective inference, precisely because an agentic system's own orchestration can now manufacture the first three at negligible cost.

\appendix

\section{Proof of Theorem~\ref{thm:non-identifiability}}
\label{app:proof-nonidentifiability}

The Bayes-optimal posterior values in the two structures of Section~\ref{sec:non-identifiability} are
\[
q_C=p
\]
and
\[
q_I=
\frac{p^2}{p^2+(1-p)^2}.
\]
Because the observable input is the same, a deterministic mechanism must choose the same value $q$ in both cases. Therefore,
\[
\max\{|q-q_C|,|q-q_I|\}
\geq
\frac{|q_I-q_C|}{2}.
\]
Since
\[
q_I-q_C
=
\frac{p(1-p)(2p-1)}
{p^2+(1-p)^2},
\]
the deterministic result follows.

For a randomized output $Q$,
\[
\mathbb E|Q-q_C|
+
\mathbb E|Q-q_I|
\geq
|q_I-q_C|,
\]
so at least one expected absolute error must be at least half the posterior separation.

\section{Synthetic Model Validation: Additional Detail}
\label{app:simulation-detail}

This appendix gives the full metrics, tables, and figures for the simulation of Section~\ref{sec:simulation}, summarized there in compressed form.

\subsection{Metrics}

We evaluate posterior RMSE, posterior standard deviation, nominal 95\% credible-interval coverage, and logarithmic score. Define the calibration ratio
\[
C=
\frac{\operatorname{RMSE}}
{\text{reported posterior standard deviation}}.
\]
A correctly calibrated Gaussian posterior has $C\approx1$; values substantially greater than one indicate overconfidence. For an aggregator with posterior density $p_A$, the negative log-likelihood (negative log score) at the realized state $\theta$ is
\[
\mathrm{NLL}_A(\theta)
=
-\log p_A(\theta).
\]
The realized log-score improvement produced by admitting a report $Z$ is
\[
\Delta\ell
=
\log P(\Theta\mid R,Z)
-
\log P(\Theta\mid R).
\]
When these probabilities are the true posteriors under the data-generating model, $\mathbb E[\Delta\ell]=I(\Theta;Z\mid R)$, an identity that provides an operational interpretation of $\Gamma(Z;R)$. Candidate aggregators are evaluated against this ideal, not assumed to satisfy the identity themselves.

\subsection{Report count versus root count}

Fixing $n=16$ and varying $k$, the information contained in the reports under the correct model is:

\begin{center}
\begin{tabular}{rc}
\toprule
Primitive roots $k$ & $I(\Theta;R)$, nats\\
\midrule
1 & 0.332\\
2 & 0.511\\
4 & 0.718\\
8 & 0.923\\
16 & 1.099\\
\bottomrule
\end{tabular}
\end{center}

The naive aggregator assigns 1.099 nats of implied information in every case because it reacts only to $n=16$, not to evidential ancestry: the same number of reports ranges from 0.332 to 1.099 nats of actual information depending only on how many independent roots generated them.

\begin{figure}[h]
\centering
\includegraphics[width=.6\linewidth]{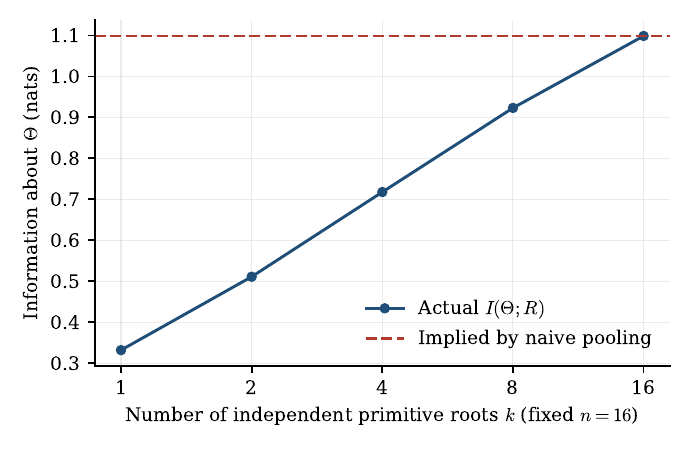}
\caption{Mutual information at fixed $n=16$ as the number of independent roots varies. Report count is constant; evidential information is not.}
\end{figure}

\begin{figure}[h]
\centering
\includegraphics[width=.6\linewidth]{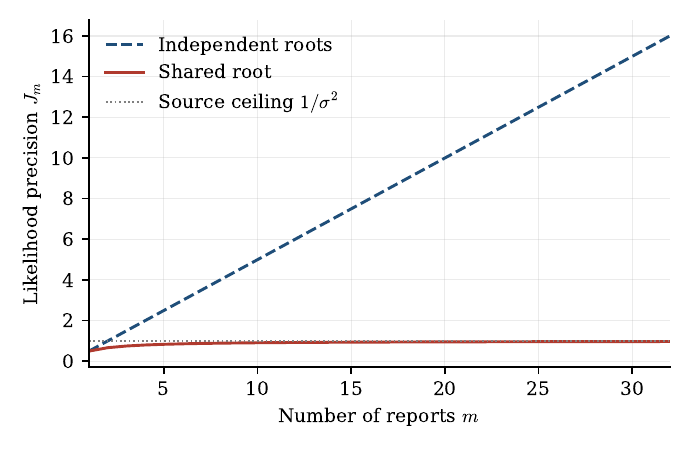}
\caption{Likelihood precision under shared-root extraction and independent-root corroboration. Shared-root precision approaches a finite ceiling, whereas independent-root precision continues to accumulate.}
\end{figure}

\subsection{Log-score deterioration}

Increasing report multiplicity with one fixed root does more than misstate uncertainty. The naive negative log score rises from approximately 1.22 at $n=1$ to 7.27 at $n=32$, while the dependence-aware Bayesian score instead improves slightly as repeated extraction removes part of the extraction noise.

\begin{figure}[h]
\centering
\includegraphics[width=.6\linewidth]{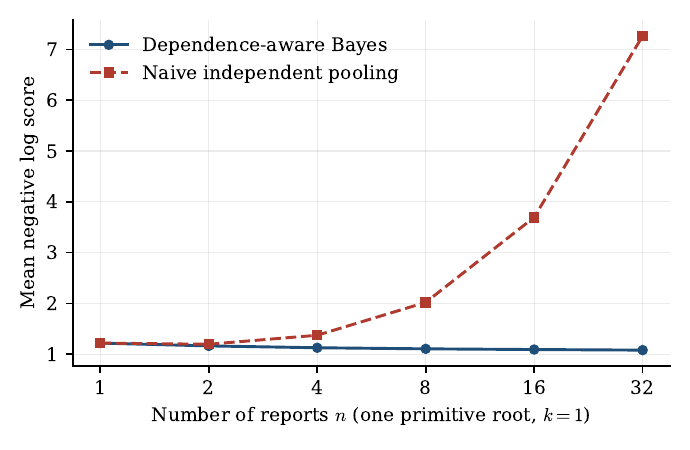}
\caption{Mean negative log score under propagation from a single evidence root. Independence-assuming aggregation deteriorates rapidly as report multiplicity increases.}
\end{figure}

\subsection{Interpretation}

The simulation should not be interpreted as empirical evidence that real LLM agents follow the Gaussian model; the dependence structure is imposed by construction. Its role is narrower. First, it numerically verifies the implementation of the analytical model, the same implementation later used to parameterize the aggregators tested in Section~\ref{sec:empirical-study}. Second, it shows that the difference between report count and evidence-root count can generate large calibration errors even in a minimal setting. Third, it provides quantitative hypotheses for the subsequent empirical experiments, in which report similarity, evidence ancestry, and report count are manipulated on real agents, independently where the design permits.

\section*{References}

\begingroup
\setlength{\parindent}{0pt}
\setlength{\parskip}{5pt}
\newcommand{\bibentry}[1]{\hangindent1.6em\hangafter1\noindent #1\par}

\bibentry{Berriaud, D., and Wattenhofer, R. (2026). Breaking the Illusion of Artificial Consensus: Clone-Robust Weighting for Arbitrary Metric Spaces. arXiv:2602.24024.}

\bibentry{Blackwell, D. (1953). Equivalent comparisons of experiments. Annals of Mathematical Statistics, 24(2), 265--272.}

\bibentry{Bovens, L., and Hartmann, S. (2002). Bayesian Networks and the Problem of Unreliable Instruments. Philosophy of Science, 69(1), 29--72.}

\bibentry{Burnat, F. A. D., and Davidson, B. I. (2026). Quotient Semivalues for False-Name-Resistant Data Attribution. arXiv:2605.07663.}

\bibentry{Cheney, J., Ahmed, A., and Acar, U. A. (2011). Provenance as dependency analysis. \textit{Mathematical Structures in Computer Science}, 21(6), 1301--1337. https://doi.org/10.1017/S0960129511000211.}

\bibentry{Clemen, R. T., and Winkler, R. L. (1999). Combining probability distributions from experts in risk analysis. Risk Analysis, 19(2), 187--203.}

\bibentry{Clemen, R. T., and Winkler, R. L. (2007). Aggregating probability distributions. In W. Edwards, R. F. Miles Jr., and D. von Winterfeldt (Eds.), Advances in Decision Analysis. Cambridge University Press.}

\bibentry{Das Sarma, A., Dong, X. L., and Halevy, A. (2011). Data Integration with Dependent Sources. Proceedings of the 14th International Conference on Extending Database Technology (EDBT), 401--412.}

\bibentry{DeMarzo, P. M., Vayanos, D., and Zwiebel, J. (2003). Persuasion bias, social influence, and unidimensional opinions. Quarterly Journal of Economics, 118(3), 909--968.}

\bibentry{Dietrich, F., and List, C. (2004). A model of jury decisions where all jurors have the same evidence. Synthese, 142(2), 175--202. https://doi.org/10.1007/s11229-004-1276-z.}

\bibentry{Douceur, J. R. (2002). The Sybil attack. In Peer-to-Peer Systems, IPTPS 2002, Lecture Notes in Computer Science 2429, 251--260.}

\bibentry{Enke, B., and Zimmermann, F. (2019). Correlation neglect in belief formation. Review of Economic Studies, 86(1), 313--332.}

\bibentry{Golub, B., and Jackson, M. O. (2010). Naïve learning in social networks and the wisdom of crowds. American Economic Journal: Microeconomics, 2(1), 112--149.}

\bibentry{Julier, S. J., and Uhlmann, J. K. (1997). A non-divergent estimation algorithm in the presence of unknown correlations. Proceedings of the American Control Conference, 2369--2373.}

\bibentry{Kish, L. (1965). Survey Sampling. Wiley.}

\bibentry{Kong, Y., and Schoenebeck, G. (2019). An Information Theoretic Framework for Designing Information Elicitation Mechanisms That Reward Truth-telling. \textit{ACM Transactions on Economics and Computation}, 7(1), Article 2, 1--33. https://doi.org/10.1145/3296670.}

\bibentry{Kuncheva, L. I., Whitaker, C. J., Shipp, C. A., and Duin, R. P. W. (2003). Limits on the Majority Vote Accuracy in Classifier Fusion. Pattern Analysis and Applications, 6(1), 22--31.}

\bibentry{Lee, D., and Kim, J. (2026). A Failure-Mode Benchmark for Polymorphic Sybil Poisoning in RAG. arXiv:2607.03739.}

\bibentry{Louck, Y. (2026). Securing LLM-Agent Long-Term Memory Against Poisoning: Non-Malleable, Origin-Bound Authority with Machine-Checked Guarantees. arXiv:2606.24322.}

\bibentry{McLaughlin, S. P., Evans, R. J., and Krishnamurthy, V. (2003). Data incest removal in a survivable estimation fusion architecture. Proceedings of the Sixth International Conference of Information Fusion.}

\bibentry{Olsson, E. J. (2002). Corroborating Testimony, Probability and Surprise. British Journal for the Philosophy of Science, 53(2), 273--288.}

\bibentry{Pedroso, L., Batista, P., and Heemels, W. P. M. H. (2026). Overlapping Covariance Intersection: Fusion with Partial Structural Knowledge of Correlation from Multiple Sources. arXiv:2603.16768.}

\bibentry{Ramchurn, S. D., Huynh, T. D., Wu, F., Ikuno, Y., Flann, J., Moreau, L., Fischer, J. E., Jiang, W., Rodden, T., Simpson, E., Reece, S., Roberts, S., and Jennings, N. R. (2016). A Disaster Response System based on Human-Agent Collectives. Journal of Artificial Intelligence Research, 57, 661--708.}

\bibentry{Sherman, J., and Morrison, W. J. (1950). Adjustment of an Inverse Matrix Corresponding to a Change in One Element of a Given Matrix. The Annals of Mathematical Statistics, 21(1), 124--127. https://doi.org/10.1214/aoms/1177729893.}

\bibentry{Williams, P. L., and Beer, R. D. (2010). Nonnegative Decomposition of Multivariate Information. arXiv:1004.2515.}

\bibentry{Winkler, R. L. (1981). Combining probability distributions from dependent information sources. Management Science, 27(4), 479--488.}

\bibentry{Wu, Z. (2026). Split the Labor: Separating Evidence Interpretation from Decision Aggregation. arXiv:2608.14509.}

\bibentry{Xu, H. (2026). Mesh Inference: A Formal Model of Collective Inference Without a Center. arXiv:2606.19537.}

\bibentry{Yan, C., Yue, Z., Zhao, F., Lin, E., Jia, L., Tong, H., Lyu, M., Sun, C., and Zeng, Y. (2026). When Truth Is Distributed: Misinformation Derails Collective Fact Recovery in LLM-Based Multi-Agent Systems. arXiv:2608.03421.}

\endgroup

\end{document}

%% file: draft_empirical_section.tex
\section{Empirical Study with LLM Agents}
\label{sec:empirical-study}

The simulation study of Section~\ref{sec:simulation} verifies the
analytical model under an imposed dependence structure. This section tests
the same predictions against real language-model agents, where the
dependence structure is not imposed but measured. Three experiments are
reported. Grid~A and Grid~B are confirmatory: their design, thresholds,
and validity filter were frozen before evaluation data were collected,
following the pilot procedure described below. A correlated-extraction
refinement of the shared-root model is exploratory: the specification was
present in the theoretical model prior to the experiment, but the decision
to evaluate it empirically was prompted by an observed failure of the
simpler independent-extraction specification, so its parameters, though
estimated exclusively on data disjoint from every evaluation set, were
chosen after seeing evaluation behavior. A $2\times2$
similarity-versus-ancestry design is confirmatory with respect to its own
frozen design, which was fixed after a pilot and a disjoint calibration
run and before its evaluation data were collected.

\subsection{Experimental design and calibration}
\label{sec:emp-design}

Agents are a single small language model (Haiku-class) reading short
synthetic memos describing a fictional company's quarterly revenue. Each
memo states three figures: one segment directly, one segment as a
percentage of the first, and one segment as growth over a stated
prior-quarter figure; the agent must combine these to estimate total
revenue $\hat E$, which is never stated directly. Every displayed figure
recombines exactly to the primitive evidence value $E_{wj}$ by
construction, so extraction noise (the extraction error $\eta_i$ of
Section~\ref{sec:shared-root-model}) reflects the agent's behavior rather
than residual rounding. The latent state $\Theta_w$ for each fictional
world is drawn from $\mathcal N(500,100^2)$ truncated to positive values,
and the primitive evidence deviates from $\Theta_w$ by
$\varepsilon\sim\mathcal N(0,\sigma^2)$ with $\sigma=50$.

A disjoint set of 100 calibration worlds, never used for evaluation,
estimates the parameters that evaluation depends on. On this split, the
document-level distortion is $\hat\sigma^2=2518.7$ and the extraction
variance is $\hat\nu^2=5107.8$, giving an intraclass correlation
$\hat\rho=0.330$, comfortably inside the workable band identified by a
pilot run and consistent across independent draws (a pilot estimate on 30
separate worlds gave $\hat\rho$ in the range 0.25 to 0.33). A leakage
control, asking the model to estimate revenue from the company name alone
with no document, found no signal: the no-document estimate's RMSE was 5.1
times the constant-prior baseline's, and the correlation between
no-document estimates and the true state was $0.165$ at the world level
(100 world clusters, each cluster's three no-document calls averaged
before computing the correlation to avoid pseudo-replication), below the
pre-specified threshold of 0.20.

Two aggregators are compared throughout. The naive aggregator pools $n$
reports as conditionally independent observations of variance
$\hat\sigma_r^2=\mathrm{Var}(R-\Theta)=7718.4$, estimated on the same
calibration split. The provenance-aware aggregator uses the exact
shared-root precision of Proposition~\ref{prop:saturation},
$J_m=m/(\hat\nu^2+m\hat\sigma^2)$, combined additively across roots when
more than one is present. A report-space deduplication baseline embeds
each report's rationale text, clusters by a cosine-similarity threshold,
and pools cluster means with the naive method; its threshold was selected
on calibration data by minimizing mean negative log score and frozen
before evaluation.

\subsection{Report multiplicity without evidence multiplicity}
\label{sec:emp-multiplicity}

The first question is simple: does multiplying reports from the same
evidence make the system more confident than it should be? Holding
primitive evidence fixed while increasing the number of derived
reports should increase confidence faster than predictive performance in
aggregators that assume independence, the central prediction tested here;
this report-multiplicity design is Grid~A, the first of the two
confirmatory grids. Across 300 evaluation worlds, with
$n\in\{1,2,4,8,16,32\}$ reports drawn from a single root's report pool and
nested so that each smaller $n$ is a prefix of the next, naive coverage of
a nominal 95\% interval falls from 0.940 at $n=1$ to 0.263 at $n=32$, and
the calibration ratio $C=\mathrm{RMSE}/\text{posterior s.d.}$ rises from
0.961 to 5.473. The provenance-aware aggregator's coverage instead stays
within a much narrower 0.850 to 0.940 band across the same range, and its
calibration ratio stays under 1.46 throughout. Naive mean negative log
score more than triples over this range, from 5.57 to 18.63, while the
provenance-aware score is nearly flat, from 5.57 to 5.81. Figure~\ref{fig:emp-calibration}
shows the calibration ratio and Figure~\ref{fig:emp-logscore-n} the log
score across $n$.

\begin{figure}[t]
\centering
\includegraphics[width=.68\linewidth]{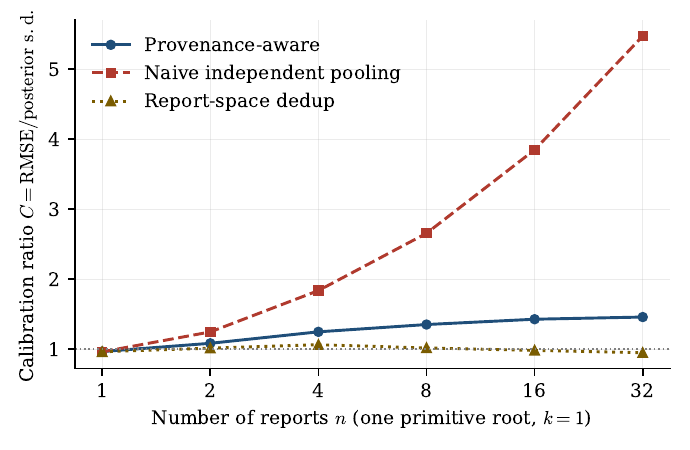}
\caption{Empirical calibration ratio as report count increases with one
primitive evidence root ($k=1$), 300 evaluation worlds. The
provenance-aware posterior stays within a narrow band; the
independence-assuming posterior becomes progressively overconfident. The
report-space dedup baseline is shown for reference; see
Section~\ref{sec:emp-2x2} for its behavior under manipulated
representation.}
\label{fig:emp-calibration}
\end{figure}

\begin{figure}[t]
\centering
\includegraphics[width=.68\linewidth]{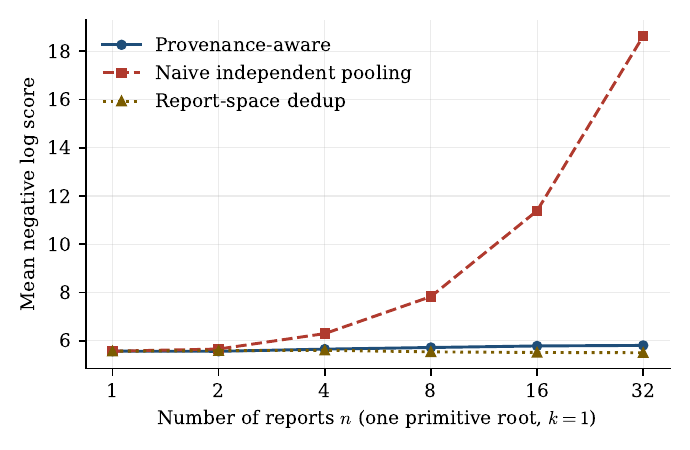}
\caption{Mean negative log score against report count $n$, one primitive
evidence root, 300 evaluation worlds.}
\label{fig:emp-logscore-n}
\end{figure}

The residual drift in provenance-aware coverage away from the nominal 0.95
is not attributable to the shared-root correction itself;
Section~\ref{sec:emp-gamma} below identifies its source directly.

\subsection{Independent evidence roots restore information}
\label{sec:emp-roots}

The second question is the mirror image: when additional reports really
do carry new evidence, does the overconfidence disappear? Holding report
count fixed at $n=16$ and varying the number of independent
primitive roots $k\in\{1,2,4,8,16\}$ isolates report count from evidential
ancestry, the empirical analogue of Section~\ref{sec:sim-rootcount}'s
simulation; this root-multiplicity design is Grid~B, the second
confirmatory grid. At $k=1$ the two aggregators are far apart: naive coverage is
0.400 against provenance-aware's 0.860, and naive mean negative log score
is 11.39 against provenance-aware's 5.79. As $k$ rises toward $n$, this
gap closes monotonically, and at $k=16$, where every report has its own
root and the naive independence assumption is literally correct, the two
aggregators are statistically indistinguishable: coverage 0.927 against
0.927, mean negative log score 4.584 against 4.585, RMSE 23.49 against
23.49 with overlapping bootstrap intervals. The $k=1$ cell of this design
reproduces Grid~A's $n=16$ cell exactly, since it is the same underlying
data viewed from the other axis, which serves as an internal consistency
check on the analysis pipeline. Figure~\ref{fig:emp-logscore-k} shows mean
negative log score against $k$.

\begin{figure}[t]
\centering
\includegraphics[width=.68\linewidth]{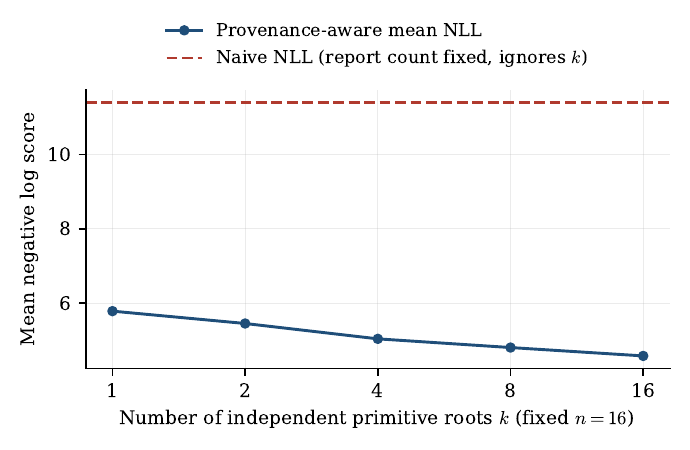}
\caption{Mean negative log score at fixed report count $n=16$ as the
number of independent primitive roots $k$ varies, 300 evaluation worlds.
Although the naive aggregator uses the same independence-based
uncertainty model for every $k$, its realized log score improves as
additional independent roots improve the accuracy of the pooled estimate;
what does not track $k$ is its assumed uncertainty, which is why its
calibration lags provenance-aware's until $k$ approaches $n$.}
\label{fig:emp-logscore-k}
\end{figure}

This result functions as a structural control on
Section~\ref{sec:emp-multiplicity}'s finding. The naive aggregator is not
simply penalized for having many reports; it is penalized specifically for
treating correlated reports as independent, and the penalty vanishes
exactly when that assumption becomes true. Report count alone does not
predict the naive-provenance gap; root count does.

\subsection{Correlated extraction errors and the information ceiling}
\label{sec:emp-gamma}

The third question concerns re-reading: how much can repeated extraction
from one fixed document actually add, and where does it stop? When
several agents independently extract information from a fixed noisy
source, additional agents can improve inference, but improvement should
exhibit diminishing returns toward a source-dependent ceiling; the
relevant empirical target is the existence and approximate level of that
ceiling. The independent-extraction specification underlying the
provenance-aware aggregator above assumes each report's extraction error
is drawn afresh. Grid~A's own data reject this assumption. The variance of the $m$-report
block mean, $\mathrm{Var}(\bar R_m-\Theta)$, computed across Grid~A's 300
worlds for $m\in\{1,2,4,8,16,32\}$, declines only slightly, from 9177 at
$m=1$ to 7510 at $m=32$, while the independent-extraction prediction
$\hat\sigma^2+\hat\nu^2/m$ falls from 7627 to 2678 over the same range and
lies outside the block mean's bootstrap confidence interval at every
$m\geq2$. Correlated replicate errors, the mechanism
Section~\ref{sec:correlated-extraction} anticipates for agents sharing a
base model, are the natural candidate: under that specification the
block-mean variance is
$\hat\sigma^2+\hat\gamma\hat\nu^2+(1-\hat\gamma)\hat\nu^2/m$ for a
correlation parameter $\hat\gamma$. This specification was already present
in the theoretical model; the decision to test it empirically was prompted
by the independent model's failure above, so the check that follows is
reported as an out-of-sample validation of an exploratory choice, not a
pre-registered confirmatory test.

Estimating $\hat\gamma$ exclusively on the 100 calibration worlds, from
the ratio of mean within-document report variance to $\hat\nu^2$, gives
$\hat\gamma_{\mathrm{cal}}=0.719$. The resulting prediction, with every
parameter estimated exclusively from the disjoint calibration split and
without fitting to Grid~A, falls inside Grid~A's bootstrap interval at all
six values of $m$, while the independent prediction remains outside for
$m\geq2$. An in-sample estimate of $\gamma$ from Grid~A's own
within-document variance gives 0.703, close to the calibration value and
consistent with a stable underlying parameter rather than an artifact of
either sample. Figure~\ref{fig:emp-precision} shows both predicted curves
against the empirical one.

\begin{figure}[t]
\centering
\includegraphics[width=.68\linewidth]{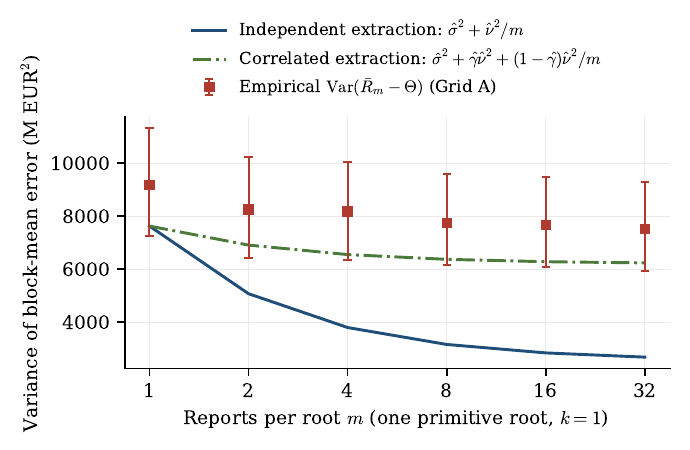}
\caption{Empirical block-mean variance curve, 300 Grid~A evaluation
worlds, against the independent-extraction prediction
$\hat\sigma^2+\hat\nu^2/m$ and the correlated-extraction prediction using
$\hat\gamma_{\mathrm{cal}}=0.719$ estimated exclusively on the disjoint
calibration split. Both predictions use only calibration-fit parameters;
neither is tuned on the plotted data.}
\label{fig:emp-precision}
\end{figure}

An aggregator built on this specification, replacing
Proposition~\ref{prop:saturation}'s block precision with the
correlated-extraction form $J_m=m/\bigl(v+(m-1)c\bigr)$ for
$v=\hat\sigma^2+\hat\nu^2$ and $c=\hat\sigma^2+\hat\gamma_{\mathrm{cal}}\hat\nu^2$,
and using only calibration-fit parameters, restores coverage to between
0.940 and 0.953 across every value of $n$ in Grid~A, with calibration
ratio between 0.95 and 0.96 throughout, compared to the plain
provenance-aware aggregator's drift down to 0.850. This aggregator is
presented as an exploratory extension, not a confirmatory result, and is
kept visually separate from the confirmatory figures:
Figure~\ref{fig:emp-gamma-aggregator} shows naive, provenance-aware, and
provenance-with-correlated-extraction coverage and calibration ratio
together. A direct sensitivity check, subtracting the calibration-split
report bias of $-11.3$ from every report and re-running every aggregator,
changes no evaluation cell materially (the $k=16$ coverage in Grid~B moves
from 0.927 to 0.920), so the residual gap between provenance-aware
coverage and the nominal 0.95 is attributed to correlated extraction error
and heavy-tailed residuals rather than to bias. Extraction noise pooled
across Grid~A's root-0 reports departs visibly from normality in the
tails, with excess kurtosis 2.35 and 1.2\% of standardized residuals
beyond three standard deviations against 0.27\% expected under a normal
distribution (Figure~\ref{fig:emp-qq}), a further source of Gaussian
likelihood misspecification contributing to the remaining gap to nominal
coverage.

\begin{figure}[t]
\centering
\includegraphics[width=\linewidth]{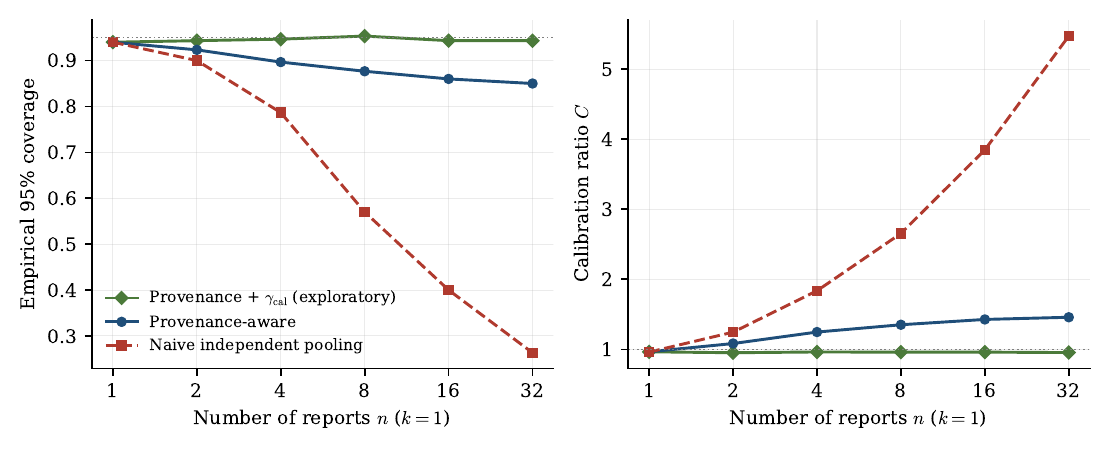}
\caption{Exploratory extension: naive, provenance-aware, and
provenance-aware with the correlated-extraction correction, coverage
(left) and calibration ratio (right) across Grid~A. Kept visually separate
from the confirmatory Figures~\ref{fig:emp-calibration} and~\ref{fig:emp-logscore-n}.}
\label{fig:emp-gamma-aggregator}
\end{figure}

\begin{figure}[t]
\centering
\includegraphics[width=.68\linewidth]{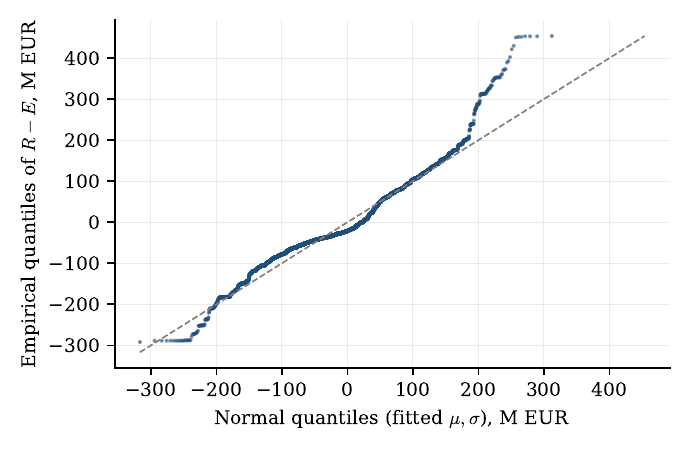}
\caption{Quantile-quantile plot of extraction noise $R-E$ against a fitted
normal distribution, pooled over Grid~A's root-0 reports ($n=9600$).}
\label{fig:emp-qq}
\end{figure}

\subsection{Representation similarity does not identify evidential ancestry}
\label{sec:emp-2x2}

The final question is whether the reports themselves can reveal where
they came from: can a system tell, by inspecting report content, whether
the reports descend from one source or several? A report-space defense
should be misled if it identifies evidential
independence from semantic similarity rather than from ancestry itself,
the possibility Theorem~\ref{thm:non-identifiability} leaves open once
provenance is unavailable. A $2\times2$ design tests this directly. Every world generates
four independent evidence roots; ancestry is manipulated by which roots
feed a cell's four reports, shared for cells A and C, independent for
cells B and D. Representation is manipulated separately and only after
elicitation, through a controlled replacement of the textual rationale
channel that leaves the estimate and every evidential input unchanged: one
frozen prompt, identical across every cell, produces an estimate and a raw
rationale; a deterministic post-hoc renderer, which never reads the
report's estimate or raw rationale, then supplies the rationale text that
is actually embedded, drawn from one of two fixed template pools that
describe the task's common three-segment structure in different
vocabulary rather than paraphrasing any individual report's specific
content. Cells A and B receive one rendering style, C and D the other. An
earlier design that manipulated representation through the elicitation
prompt itself was found in a pilot to change extraction accuracy between
styles, confounding representation with extraction quality; the redesign
used here removes that confound by construction, since the estimate cannot
depend on a rendering step that never sees it (the earlier design and why
it was rejected are recorded in the supplementary experimental log).

Across 200 evaluation worlds, the effect of the representation
manipulation on the deduplication baseline's inferred cluster count is
$+1.425$ (paired bootstrap by world, 95\% CI $[1.363,1.485]$), computed as
the average cluster count in the two dissimilar-representation cells minus
the average in the two similar-representation cells. The effect of a
fourfold change in true root count is $+0.040$ (95\% CI
$[-0.045,0.120]$), computed the same way across the ancestry contrast. The
representation effect is large and precisely estimated, whereas the
ancestry effect is close to zero and statistically indistinguishable from
zero. A separate similarity check,
$\Delta_{BC}=\mathbb E[\cos(B)]-\mathbb E[\cos(C)]$, comparing
independent-root reports rendered alike against shared-root reports
rendered differently, gives $0.248$ (95\% CI $[0.236,0.259]$) on the
evaluation set, closely reproducing a pilot estimate of $0.242$
(95\% CI $[0.216,0.270]$) obtained on a disjoint set of 30 worlds before
the design was frozen.

\begin{figure}[t]
\centering
\includegraphics[width=.7\linewidth]{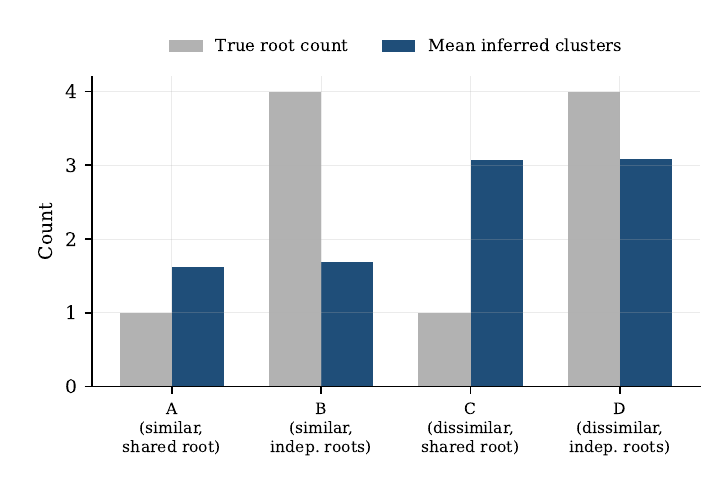}
\caption{True root count versus mean inferred cluster count for each of
the four design cells, 200 evaluation worlds. Cells sharing a
representation style (A, B) and cells sharing a dissimilar-representation
pairing (C, D) cluster together regardless of true ancestry.}
\label{fig:emp-clusters}
\end{figure}

Because cells A and C share identical evidence and identical estimates,
differing only in which representation was rendered downstream, the
contrast between them isolates representation's effect from every other
factor. Under the deduplication aggregator, this contrast moves every
downstream metric in the harmful direction with confidence intervals
excluding zero: coverage falls by $0.125$ ($[-0.175,-0.080]$), mean
negative log score rises by $0.223$ ($[0.099,0.373]$), and calibration
ratio rises by $0.416$ ($[0.333,0.513]$). The analogous contrast between
the independent-root cells, D minus B, shows no coverage or calibration
effect distinguishable from zero and a favorable shift in log score,
because for cell D the dissimilar rendering happens to coincide with the
cell's true independent structure rather than working against it. Naive
and provenance-aware aggregators are, as they should be, identical between
A and C and between B and D, since neither reads report content.

A threshold, chosen to balance mean log score equally across all four
cells on 60 calibration worlds disjoint from evaluation and frozen before
evaluation began, does not resolve the underlying conflict. The
false-merge rate is the fraction of different-root report pairs assigned
to the same cluster; the false-split rate is the fraction of same-root
report pairs assigned to different clusters. At this
threshold the false-merge rate in cell B is 0.622 and the false-split rate
in cell C is 0.846. Sweeping every cosine-similarity value that is
actually distinct in the evaluation data, so that every threshold capable
of producing a different clustering decision is covered exactly rather
than by a spaced grid, the minimum achievable value of the worse of the
two error rates is $\min_t\max\bigl(\mathrm{FMR}_B(t),\mathrm{FSR}_C(t)\bigr)=0.846$.
No threshold for this similarity-based clustering rule achieves low
false-merge and false-split error simultaneously on the observed
evaluation set. Figure~\ref{fig:emp-tradeoff} shows the false-merge and
false-split rates against each other across the swept range
(Figure~\ref{fig:emp-clusters}, above, shows mean inferred cluster count
against true root count), and Figure~\ref{fig:emp-coverage} shows coverage
by aggregator and cell.

\begin{figure}[t]
\centering
\includegraphics[width=.68\linewidth]{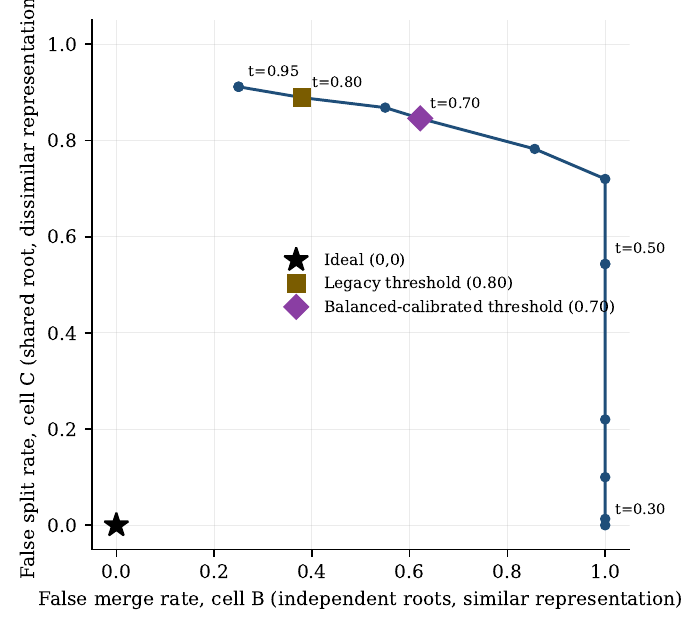}
\caption{False merge rate in cell B against false split rate in cell C
across the swept threshold range, 200 evaluation worlds. The legacy
threshold (0.80) and the balanced-calibrated threshold (0.70) are marked;
the ideal point (0,0) is far from the entire curve.}
\label{fig:emp-tradeoff}
\end{figure}

\begin{figure}[t]
\centering
\includegraphics[width=.7\linewidth]{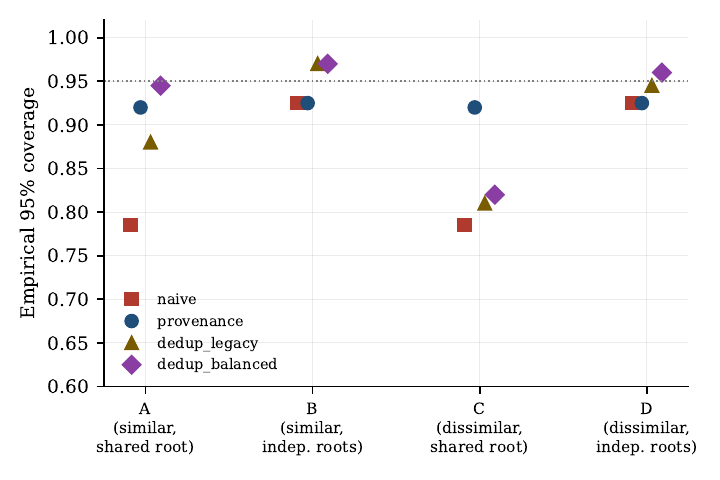}
\caption{Empirical 95\% coverage by aggregator and design cell, 200
evaluation worlds. Naive and provenance-aware are identical between cells
sharing true ancestry (A/C and B/D); the dedup aggregators are not.}
\label{fig:emp-coverage}
\end{figure}

This experiment is a controlled representation stress test: it establishes
that a controlled change to the textual rationale channel alone, holding
the latent state, the primitive evidence, the extraction process, and the
estimate fixed, can invert what report-only similarity implies about
evidential ancestry for one concrete mechanism (rationale embedding with
cosine-threshold clustering) on one task family. It illustrates the
report-only non-identifiability result of Theorem~\ref{thm:non-identifiability}
operationally; it does not establish that every report-space defense must
fail this way, nor does it characterize how often adversarial or
coincidental representation conflicts of this kind arise in deployed
systems.

\subsection{Summary of empirical findings}
\label{sec:emp-summary}

Three results, obtained from a single small language model reading
synthetic documents, are consistent with the paper's central claims.
Increasing report count at fixed evidential ancestry produces naive
overconfidence that grows with report count and is not accompanied by a
comparable accuracy gain, while a provenance-aware aggregator using the
same reports stays within a narrow calibration band; the naive-provenance
gap closes specifically as evidential ancestry, not report count, moves
toward independence. The same model's replicate extractions on a shared
document are more strongly correlated than an independent-extraction model
predicts, a pattern that a correlated-extraction refinement of the
shared-root model, evaluated out of sample, both explains and corrects. A
controlled design that manipulates representation similarity and
evidential ancestry independently shows that a concrete report-space
deduplication mechanism tracks the former far more strongly than the
latter, and that no similarity threshold resolves the resulting conflict
on the evaluation data.

These findings come with limitations that bound their scope. All data are
synthetic quarterly-revenue memos from one task family, generated to
control extraction difficulty precisely; no natural-language corpus or
real-world evidentiary claim is involved. A single main model was used
throughout, and no claim is made about how these effects scale across
model families; an alternative, larger model was evaluated and found
unsuitable for studying extraction noise at all, since it does not expose
an explicit sampling temperature and so cannot produce genuinely
independent replicate samples, a finding recorded with the full model
comparison in the supplementary experimental log rather than repeated
here. The $2\times2$ design's representation manipulation is a
deliberately artificial stress test and should not be read as a claim
about the rationale styles produced in ordinary use. The extraction-noise
distribution is heavy-tailed (excess kurtosis 2.35, with 1.2\% of
standardized residuals beyond three standard deviations against 0.27\%
under normality), so the Gaussian aggregators compared here are themselves
somewhat misspecified in the tails, a limitation that applies uniformly
across naive, provenance-aware, and correlated-extraction aggregators and
so does not favor any one of them. Finally, a pre-registered
informativeness gate from the pilot phase was not met on its literal
threshold; inspection traced the shortfall to genuine arithmetic
unreliability in the main model rather than to document ambiguity, and it
is retained here as a disclosed, uncorrected property of the studied
regime rather than treated as disqualifying, with the full gate protocol
and threshold recorded in the supplementary experimental log.

Code, the exact elicitation prompts, the frozen raw model outputs, the
calibration and evaluation world-id splits, and the analysis scripts
underlying every result in this section will be made publicly available.
The frozen Grid~A and Grid~B evaluation data are additionally packaged as
a standalone benchmark with a documented aggregator interface, a scoring
command, and reference baseline adapters, so that other dependence-aware
aggregation methods can be evaluated on the same report-multiplicity and
root-multiplicity axes without rerunning any model calls.
Every reported table and figure in this section can be reproduced from
the archived, frozen model outputs without access to the Anthropic API or
any other hosted model provider. (The synthetic-model validation of
Section~\ref{sec:simulation} involves no model calls at all and is
separately reproducible from its fixed simulation seed.) Rerunning the live elicitation calls that generated those
outputs is possible for a reader with valid API credentials but
constitutes a replication of the experiment rather than a guaranteed exact
computational reproduction, since hosted model providers can update model
weights or serving behavior behind a fixed model identifier without
notice.